\documentclass{article}

\usepackage{microtype}
\usepackage{graphicx}
\usepackage{booktabs} 
\usepackage{hyperref}

\def\w{{\bf w}}

\def\s{{\bf s}}

\def\b{{\bf b}}
\def\m{{\bf m}}
\def\EX{{\mathbb {E}}}

\usepackage[algo2e,linesnumbered,ruled,vlined,noend]{algorithm2e}
\usepackage[dvipsnames]{xcolor}
\usepackage{multirow}
\usepackage{placeins}
\usepackage{subcaption}

\usepackage[accepted]{icml2023}

\usepackage{amsmath}
\usepackage{amssymb}
\usepackage{mathtools}
\usepackage{amsthm}

\usepackage[capitalize,noabbrev]{cleveref}

\theoremstyle{plain}
\newtheorem{theorem}{Theorem}[section]
\newtheorem{proposition}[theorem]{Proposition}

\newtheorem{corollary}[theorem]{Corollary}
\theoremstyle{definition}

\theoremstyle{remark}

\usepackage[textsize=tiny]{todonotes}

\icmltitlerunning{Long Horizon Temperature Scaling}

\begin{document}

\twocolumn[
\icmltitle{Long Horizon Temperature Scaling}



\icmlsetsymbol{equal}{*}

\begin{icmlauthorlist}
\icmlauthor{Andy Shih}{stanford}
\icmlauthor{Dorsa Sadigh}{stanford}
\icmlauthor{Stefano Ermon}{stanford}
\end{icmlauthorlist}

\icmlaffiliation{stanford}{Department of Computer Science, Stanford University}
\icmlcorrespondingauthor{Andy Shih}{andyshih@cs.stanford.edu}

\vskip 0.3in
]



\printAffiliationsAndNotice{}  

\begin{abstract}
Temperature scaling is a popular technique for tuning the sharpness of a model distribution. It is used extensively for sampling likely generations and calibrating model uncertainty, and even features as a controllable parameter to many large language models in deployment.
However, autoregressive models rely on myopic temperature scaling that greedily optimizes the next token. To address this, we propose \textit{Long Horizon Temperature Scaling} (LHTS), a novel approach for sampling from temperature-scaled \textit{joint} distributions. 
LHTS is compatible with all likelihood-based models, and optimizes for the long horizon likelihood of samples. We derive a temperature-dependent LHTS objective, and show that finetuning a model on a range of temperatures produces a single model capable of generation with a  controllable long horizon temperature parameter. We experiment with LHTS on image diffusion models and character/language autoregressive models, demonstrating advantages over myopic temperature scaling in likelihood and sample quality, and showing improvements in accuracy on a multiple choice analogy task by $10\%$. Our code is available at \url{https://github.com/AndyShih12/LongHorizonTemperatureScaling}.
\end{abstract}

\section{Introduction}

\begin{figure*}[t]
    \includegraphics[width=\linewidth]{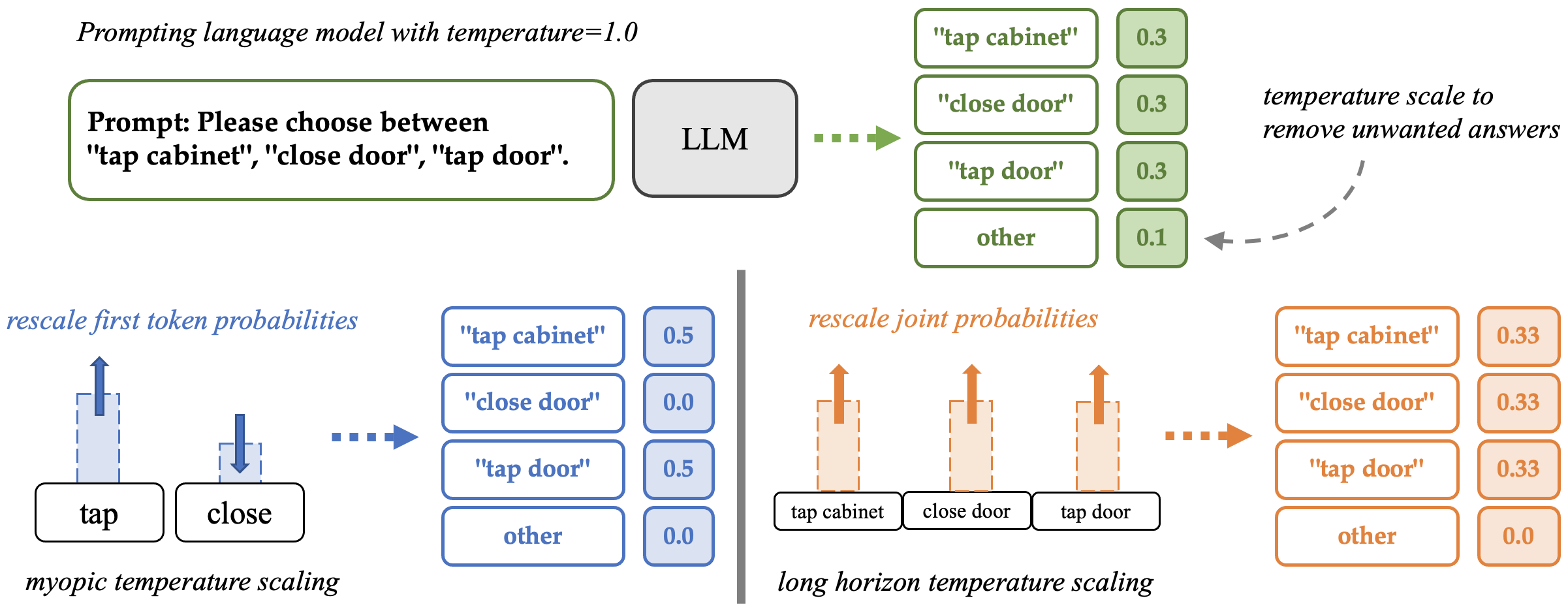}
    \caption{Pitfalls of myopic temperature scaling. At the top of the diagram, we depict prompting a language model for a choice of three actions. The language model may respond with each choice with a probability of $0.3$ (shown in green), and a remaining probability of $0.1$ of outputting irrelevant answers. To reduce the probability of irrelevant answers, we can lower the temperature of the model. In blue, we show that myopic temperature scaling will unintuitively lump the probabilities for the two actions ``tap cabinet'' and ``tap door'', because they share the same first token ``tap''. Therefore, lowering the myopic temperature will emphasize the probability on these two choices, and diminish the probability of choosing ``close door''. On the other other hand, in orange we show that long horizon temperature scaling correctly scales the joint probability of the full sequence, equally distributing a probability of one-third among the three choices.}
    \label{fig:cover}
\end{figure*}

Temperature scaling is a simple yet effective technique for rescaling model outputs: lowering the temperature to increase the probability of high-likelihood outcomes, or vice versa. In discriminative settings, tuning the temperature has shown success as a calibration method~\cite{guo2017calibration, nixon2019measuring, desai2020calibration}. The model outputs a small set of class probabilities, which can be tractably rescaled to match the desired calibration metric.

In generative tasks, temperature scaling also serves as a method for controlling the randomness of model outputs, and has shown to be useful for many natural language generation tasks such as summarization and question answering~\cite{liang2022holistic}. Many current models in deployment~\cite{brown2020language, bommasani2021opportunities} even expose the model temperature as a user-controllable parameter in their API. These autoregressive language models execute temperature scaling one token at a time, rescaling the probability of the next token to be proportional to $\log p(x_i|x_{<i})/T$. However, this mechanism is myopic, optimizing for the next token instead of the full sequence.

We reexamine the current practice of temperature scaling for generative models. Unlike discriminative tasks, generative tasks produce high-dimensional outputs. In other words, rescaling the model outputs should, in principle, rescale joint probabilities according to $\log p(x)/T$. Lowering the temperature of a language model should ideally bias the model towards generation of full text sequences with high joint likelihood, not just greedy generation of the next likely tokens. 
However, due to the intractability of joint temperature scaling, existing model families rely on various ad-hoc approximations such as myopic temperature scaling. This perspective highlights the following concerns. 
\begin{itemize}
    \item[\textbf{A)}] Current temperature scaling for autoregressive models is a myopic approximation to temperature scaling of joint probabilities.
\end{itemize}
Many other model families do not support myopic approximations, and are left with the intractable problem of joint temperature scaling. Some sidestep the problem by defining various notions of pseudo-temperatures~\cite{kingma2018glow, vahdat2020nvae}.
\begin{itemize}
    \item[\textbf{B)}] Many non-autoregressive generative models either rely on pseudo-temperatures or do not use temperature scaling altogether.
\end{itemize}

To address these concerns, we aim head-on for the goal of joint temperature scaling. Instead of handling various model-specific temperature scaling techniques, we set out to develop a practical and general mechanism for sampling from temperature-scaled joint distributions. We propose \textit{Long Horizon Temperature Scaling} (LHTS), a novel and tractable approach for sampling from a temperature-scaled joint distribution that is \textbf{A)} non-myopic and \textbf{B)} compatible with all likelihood-based generative models. LHTS requires finetuning a likelihood-based model on a temperature-dependent objective, after which the model can sample long horizon temperature-scaled outputs without any additional cost over standard sampling. By finetuning over a range of temperatures, we can learn a single model capable of generation with a controllable parameter, extrapolating even to temperatures unseen during finetuning. 

LHTS enables autoregressive models to optimize for high likelihood outputs over a long horizon instead of a single token (Figure~\ref{fig:cover}). For other likelihood-based models (e.g. VAEs, normalizing flows, diffusion models), LHTS presents a unified model-agnostic temperature scaling mechanism. We experiment with LHTS in three settings: a diffusion image model, an autoregressive character model, and autoregressive large language models. Our experiments show that LHTS can achieve a better tradeoff between likelihood and diversity compared to pseudo-temperature scaling for diffusion models, and compared to myopic temperature scaling for autoregressive models. On a downstream analogy multiple-choice task, LHTS improves the accuracy of GPT-2 by $10\%$ over myopic temperature scaling.

\section{Background}

For generative tasks, we have access to a data distribution $p_{\text{data}}(x)$ in the form of a training set $\mathcal{D}$ of i.i.d. samples, from which we aim to learn a faithful model $p(x)$ of the data distribution. In principle, having learned the ideal model $p(x)$ for our downstream task, we would be satisfied with drawing conditional/unconditional samples from $p(x)$. 

However, in practice, biasing samples towards higher likelihood regions of the model distribution is often beneficial. For example, we often choose to calibrate the entropy of a suboptimal model~\cite{holtzman2019curious}, generate less noisy behavior by taking the argmax action, or simply sample from a sharper distribution. The most prominent technique for biasing towards high likelihood regions is \textit{temperature scaling} with a scalar $T$.
\begin{align}
    \log p_T(x) = \log p(x) / T  - \log Z_{p_T}\label{eq:temp_joint}
\end{align}

where $Z_{p_T}$ is the partition function.
For temperatures $T < 1$, the scaled model $p_T(x)$ defines a sharper distribution, which is useful for a variety of applications mentioned above.

\subsection{Myopic temperature scaling}
Autoregressive models, such as GPT~\cite{radford2019language,brown2020language}, implement a myopic approximation to temperature scaling. Autoregressive models learn a set of univariate conditional distributions $\log p(x_i | x_{<i})$ and rely on the factorization of the joint distribution via chain rule $\log p(x) = \sum_i \log p(x_i | x_{<i})$. When sampling with a temperature $T$, they rescale each univariate conditional by $T$.
\begin{align}
    \log p^{\text{myopic}}_T(x_i | x_{<i}) = \log \frac{ e^ {\log p(x_i | x_{<i}) / T} } { \sum_k e^{\log p(x_i = k | x_{<i}) / T} }
\end{align}
This approach is efficient since it handles one dimension at a time and only requires rescaling the output logits. However, since the scaling is \textit{myopic}, the chain rule factorization does not preserve the scaled joint distribution in Eq~\ref{eq:temp_joint}. 
\begin{align}
    \log p_T(x) \neq \sum_i \log p^{\text{myopic}}_T(x_i | x_{<i})
\end{align}

It is easy to see that in the extreme case, myopic scaling of an autoregressive model with $T \rightarrow 0$ will not necessarily produce the argmax sample of the joint distribution. 

\subsection{Pseudo-temperature scaling}

Non-autoregressive models are often associated with various ad-hoc notions of pseudo-temperature scaling. For example, some latent variable models~\cite{kingma2018glow, vahdat2020nvae} rescale the variance of the prior of the latent variable. However, these notions of pseudo-temperature are often model-specific, and have an unclear relationship to temperature scaling of the data likelihood.

\section{Related Work}

Temperature scaling is an effective method for calibration in discriminative settings~\cite{guo2017calibration, nixon2019measuring, desai2020calibration}, where the output predictions of a model can be rescaled post-hoc. In generative settings, such as natural language generation, myopic temperature scaling serves as an important knob for controlling the randomness of autoregressive models, often featuring as a user-controllable parameter in deployment~\cite{brown2020language, bommasani2021opportunities}. For latent variable models, such as normalizing flows or VAEs, reducing the variance of the prior during sampling has been explored as a pseudo-temperature mechanism~\cite{kingma2018glow, vahdat2020nvae}.
Due to the high-dimensional output space of generative tasks, however, these above methods are approximations that do not directly scale the temperature of the joint distribution, and are typically model-specific. Compared to these methods, LHTS presents a unified and tractable mechanism for temperature scaling of the joint distribution.

Other techniques for post-hoc manipulation of autoregressive model generation include top-k~\cite{fan2018hierarchical} or nucleus sampling~\cite{holtzman2019curious}. More intensive search-based alternatives are also popular, such as beam search~\cite{li2016simple, vijayakumar2018diverse} for pick out high-likelihood generations. In terms of computational cost, LHTS only requires a one-time finetuning of the model, after which long horizon temperature-scaled outputs can be generated directly without search.

Biasing the model towards higher-likelihood samples can also be viewed as controllable generation. Some relevant works include Quark~\cite{lu2022quark}, which partitions the dataset based on a control signal of interest (e.g. toxicity), and reinforces the model with its own generations. Other works on controllable generation include class-conditional generation, for example with diffusion models for images~\cite{nichol2021improved}.

Finally, LHTS relates closely to amortized inference~\cite{gershman2014amortized}, since we learn a model to predict intractable temperature-scaled joint distributions. As the temperature approaches zero, LHTC approximates MAP inference~\cite{koller2009probabilistic}. 

\section{Long Horizon Temperature Scaling}

We propose \textit{long horizon temperature scaling} (LHTS), a general method to temperature scale the joint distribution of likelihood based models. LHTS proceeds by directly learning a model $q_T$
to match the temperature scaled distribution in Eq.~\ref{eq:temp_joint}. The model $q_T$ should have tractable likelihood and sampling, but typically this is satisfied by choosing the same model family as $p$, or even finetuning from $p$.
\begin{align*}
    \min_{q_T} KL(p_T || q_T) &= \min_{q_T} \EX_{x \sim p_T} [ \log p_T(x) - \log q_T(x) ] \\
    &= \min_{q_T} \EX_{x \sim p_T} [ - \log q_T(x) ]
\end{align*}

Although we don't have sample access to $p_T$, we can appeal to importance sampling from $p$.
\begin{align}
    & \EX_{x \sim p_T} [ - \log q_T(x) ] \nonumber \\
    =& \EX_{x \sim p} \frac{e^{ (\log p(x) / T) - \log Z_{p_T}}}{p(x)} [ - \log q_T(x) ] \nonumber \\
    =& \EX_{x \sim p} e^{ \frac{1-T}{T} \log p(x) - \log Z_{p_T}} [ - \log q_T(x) ] \label{eq:iw_obj_init}
\end{align}

Optimizing $q_T$ with Eq.~\ref{eq:iw_obj_init} will give us the desired temperature scaled distribution from Eq.~\ref{eq:temp_joint}, although the variance of the loss can be high due to the importance weights.

We note that the intractable constant $\log Z_{p_T}$ can be ignored since it evaluates as a constant multiplicative factor of the entire expression. More importantly, though, the same insight allows us to subtract an arbitrary data-independent \textit{baseline} $b$ for variance reduction. 
Since the importance weights are not in log-space, we need to carefully choose a baseline to keep the weights within a manageable range. We opt for keeping the weights close to $1$ by matching the empirical mean of the exponent.
\begin{align}
    b = \frac{1}{|\mathcal{D}|} \sum_{x \in \mathcal{D}} \frac{1-T}{T} \log p(x) \label{eq:baseline_b}
\end{align}

Put together, the loss for training $q_T$ can be understood as a reweighing of data by a factor $w_T(x)$ based on the temperature-scaled joint probabilities.
\begin{align}
    w_T(x) &= \exp( \frac{1-T}{T} \log p(x) - b) \\
    \mathcal{L}(q_T) &= - \EX_{x \sim p} [ w_T(x) \log q_T(x) ] \label{eq:iw_loss}
\end{align}

\begin{corollary}
$\frac{e^b}{Z_{p_T}} \mathcal{L}(q_T) = KL(p_T || q_T) + H(p_T)$.
\end{corollary}
\begin{proof}
Evaluating $\frac{e^b}{Z_{p_T}} \mathcal{L}(q_T)$ gives Eq.~\ref{eq:iw_obj_init}, which is equal to $KL(p_T || q_T) + H(p_T)$.
\end{proof}

The idea of LHTS is to train a model $q_T$ with tractable sampling on the objective in Eq.~\ref{eq:iw_loss}, so that we can sample from $q_T \approx p_T$ efficiently after training.
In this sense, LHTS can be considered an amortized inference method for accessing otherwise intractable temperature-scaled joint distributions.
Compared to myopic temperature scaling, LHTS is not a pure post-hoc transformation since it requires model learning. Nevertheless, we can avoid learning completely from scratch, by finetuning $q_T$ from $p $ (which can be thought of as $q_{T=1}$). In return for the cost of finetuning, LHTS improves upon myopic temperature scaling in two ways. First, the temperature operates on the joint (long horizon) distribution, instead of greedily on one dimension at a time. Second, LHTS can be readily applied to any likelihood-based generative model, beyond just autoregressive models.

In the rest of this section, we examine LHTS on hierarchical latent variable models and autoregressive models. 

\subsection{LHTS on Hierarchical Latent Variable Models}
Applying LHTS on hierarchical latent variable models is straightforward, by using their variational lower bound estimates of the data likelihood.
\begin{align}
    \log p(x_0) \geq \EX_{h} \Big[ & D_{KL}( h(x_K | x_0 ) || p(x_K)) - \log p(x_0|x_1) \nonumber \\
                   + \sum_{k>1} & D_{KL}( h(x_{k-1} | x_k,  x_0 ) || p(x_{k-1} | x_k )) \Big] \label{eq:diffusion_elbo}
\end{align}
We can then plug in this likelihood lower bound to LHTS to compute the importance weights for each data point, and finetune $q_T$ with Eq.~\ref{eq:iw_loss}, where the inner likelihood is again evaluated with the lower-bound in Eq.~\ref{eq:diffusion_elbo}.

\textbf{Diffusion Models} Although diffusion models can also be formulated as a hierarchical latent variable model, they are often trained using a simpler MSE loss on the noise~\cite{ho2020denoising}. Nevertheless, LHTS is still directly applicable by scaling the loss for each point by the importance weight.
\begin{align}
    \mathcal{L} (q_T) = &  \label{eq:ddpm_loss} \\
    \EX_{k, x_0, \epsilon} & \Big[ w_T(x_0) || \epsilon  - \epsilon_{q_T}( \sqrt{ \bar{\alpha}_k} x_0 + \sqrt{ 1 - \bar{\alpha}_k} \epsilon, k) ||^2 \Big] \nonumber 
\end{align}

We can apply LHTS in exactly the same way for other likelihood-based models by scaling the log-likelihood loss of each datapoint by its importance weight. 
For autoregressive models, however, we can take advantage of the autoregressive factorization to derive a variance-reduced formulation of LHTS, which we describe next.

\subsection{Variance-Reduced LHTS on Autoregressive Models}
\label{sec:variance-reduced-lhts}

To apply LHTS to autoregressive models, we first rewrite the LHTS objective from Eq.~\ref{eq:iw_loss} into a form that is amenable to autoregressive architectures by first sampling the index $i$ uniformly, then the prefix $x_{<i}$, and then the suffix $x_{\geq i}$.
\begin{align*}
    & - \EX_{x \sim p} [ w_T(x) \log q_T(x) ]\\
    =& - \EX_{x \sim p} [ \sum_i w_T(x) \log q_T(x_i | x_{<i}) ]\\
    =& - \EX_{i, x_{<i} \sim p} \EX_{ x_{\geq i} \sim p(\cdot | x_{<i})} [ w_T(x) \log q_T(x_i | x_{<i}) ]
\end{align*}

The purpose of this roundabout rewriting of the expectation is to illustrate that the autoregressive objective is composed of many univariate conditional losses, for each index $i$ and prefix $x_{<i}$. This derivation allows us to design the baseline more carefully, since we can choose a different baseline for each univariate conditional loss while still trivially preserving the strict properness of the overall loss function.

\begin{proposition} \label{prop:baseline}
    Let $\mathcal{L}^{\text{AR}}(q_T) = $
\begin{align*}
    - \EX_{i, x_{<i} \sim p} e^{-b(x_{<i})} \EX_{ x_{\geq i} \sim p(\cdot | x_{<i})} [ w_T(x) \log q_T(x_i | x_{<i}) ]
\end{align*}
If $b(x_{<i})$ is finite for all $x_{<i}$, then $\mathcal{L}^{\text{AR}}_{q_T}$ is a strictly proper loss function, i.e. the unique global optimum is $q_T = p_T$.
\end{proposition}
\begin{proof}
Each inner expectation takes on an importance-weighted log loss of the univariate conditional, corresponding to optimizing $KL(p_T(\cdot | x_{<i}) || q_T(\cdot | x_{<i}))$. Since an autoregressive model fits all the univariate conditionals separately, these are independent optimization problems each with strictly proper losses. Any positive combination ($b(x_{<i})$ is finite) preserves strict properness of the loss.
\end{proof}

In particular, we can set $b(x_{<i}) = \frac{1-T}{T} \log p(x_{<i}) + b(i) - b$ to be the temperature scaled joint distribution of the prefix, giving us a variance-reduced importance weight.
\begin{align}
    & - \EX_{i, x_{<i} \sim p} e^{-b(x_{<i})} \EX_{ x_{\geq i} \sim p(\cdot | x_{<i})} [ w_T(x) \log q_T(x_i | x_{<i}) ] \nonumber \\
    = & - \EX_{i, x_{<i} \sim p} e^{ \frac{1-T}{T} \log p(x_{<i}) - b(x_{<i}) } \nonumber \\
        & \EX_{ x_{\geq i} \sim p(\cdot | x_{<i})} [ e^{ \frac{1-T}{T} \log p(x_{\geq i} | x_{<i}) - b } \log q_T(x_i | x_{<i}) ] \nonumber \\
    = & - \EX_{ i, x \sim p } [  e^{ \frac{1-T}{T} \log p(x_{\geq i} | x_{<i}) - b(i) } \log q_T(x_i | x_{<i}) ] \label{eq:ar_loss}
\end{align}

Compared to Eq.~\ref{eq:iw_loss}, in Eq.~\ref{eq:ar_loss} we modified the expression in the exponent of the importance weight from $\log p(x)$ to $\log p(x_{\geq i} | x_{<i})$. This makes sense intuitively: once we have fixed a prefix $x_{<i}$ of the sequence, we only need to learn how likely a suffix should be relative to other suffixes, so we can ignore the probability of the prefix $p(x_{<i})$. Moreover, appealing to Proposition~\ref{prop:baseline}, we transformed the term $b$ to an index-dependent term $b(i)$. In a similar spirit to Eq.~\ref{eq:baseline_b}, we will set $b(i)$ to keep the weights close to $1$ by matching the empirical mean of the suffix log-likelihoods.
\begin{align}
    b(i) = \frac{1}{|\mathcal{D}|} \sum_{x \in \mathcal{D}} \frac{1-T}{T} \log p(x_{\geq i} | x_{<i})
    \label{eq:baseline_suffix}
\end{align}

\paragraph{Computing Suffix Likelihoods} One important consideration is the efficient implementation of variance-reduced LHTS on modern causal architectures of autoregressive models. Conveniently, we can vectorize the computation of suffix log-likelihoods $v_i(x) = \log p(x_{\geq i} | x_{<i}) $ via a reverse cumulative sum on the vector of univariate conditionals $u_i = \log p(x_{i} | x_{<i})$.

\paragraph{Suffix Horizon Length} Even with the above baseline, the variance of joint likelihoods can still grow quickly when the sequence length is long, e.g. $1024$. A practical approach to reducing the variance even more is by limiting the horizon to some length $h$. This means replacing all the suffix log-likelihoods $\log p(x_{\geq i} | x_{<i})$ with a horizon-bounded suffix log-likelihood $\log p(x_{i:k} | x_{<i})$ where $k = \min(i+h,\text{context length})$.

\section{Implementation}

In this section, we describe a list of practical considerations for implementing LHTS, and include concrete pseudocode for our implementation.

\paragraph{Clipping}
Even with a baseline to keep the exponents small, the importance weights still involve exponentiation. Therefore, the weights can become unstable when the log probabilities are much higher than the baseline or when the long horizon temperature is small. Therefore, we clip the log of the importance weights, introducing bias but reducing variance to help stabilize training.

\paragraph{Data Sampling}
The LHTS objective is written as an expectation over samples from $p$. We can indeed sample from $p$ in the training loop, although this empirically slowed down training by around a factor of $3$ for autoregressive language models. In practice, we can assume that $p$ is close to the data distribution $p_{\text{data}}$, and evaluate the LHTS objective using the training set $\mathcal{D}$. The weights of samples from $\mathcal{D}$ are then computed using $p$, which is faster than sampling from $p$.

\paragraph{Multi-Temperature Finetuning}
In some of the experiments, we finetune a single weight-tied model on a set of discrete temperatures $T_1 \ldots T_k$. Due to the differences in importance weights, more extreme temperatures incurred higher loss and hindered the training of other temperatures. Hence, we normalize the loss of each temperature to help with balanced training across the different temperatures.

\paragraph{KL Loss} Following design choices of Quark~\cite{lu2022quark}, we include a KL loss to avoid diverging from the base model $p$ too much. However, empirically we did not observe differences from the inclusion of this auxiliary loss.

\paragraph{Streaming Statistics}
As written in Eq.~\ref{eq:baseline_b} and Eq.~\ref{eq:baseline_suffix}, we choose the baseline to be the empirical mean of the data (suffix) log-likelihood. In practice, since the dataset could be very large (e.g. OpenWebText), we instead use the running mean of the data statistics as the baseline.

\subsection{Pseudocode}

Putting all the implementation details together, we present the pseudocode for LHTS finetuning in Alg.~\ref{alg:train}. We only present the variance-reduced LHTS for autoregressive models, since LHTS for diffusion models takes on a simpler form. The horizon likelihood is computed in lines 4\&5, where \texttt{RevCumSum} computes reverse cumulative sum, and \texttt{pad}$(\s_{h:}, {\bf 0}_{:h})$ appends a vector of $h$ zeros to the end of $\s_{h:}$. Importantly, the weight in line 8 is tailored to each index of the context window, using the formulation derived from Section~\ref{sec:variance-reduced-lhts}.
On line 10, \texttt{StopGradient} prevents the loss normalization calculations from affecting the gradient computation. The algorithm outputs parameters for a weight-tied model for sampling from multiple long horizon temperatures. In practice, designing the temperature embeddings to be linearly constrained (Section~\ref{sec:char-model}) even allows extrapolation to temperatures unseen during training.

\newcommand\mycommfont[1]{\footnotesize\rmfamily\color{MidnightBlue}{#1}}
\SetCommentSty{mycommfont}

\SetKwInput{KwInput}{Input}                
\SetKwInput{KwConstraint}{Constraint}      
\SetKwInput{KwOutput}{Output}              
\SetKwInput{KwReturn}{Return}              

\IncMargin{1.5em}

\begin{algorithm2e}[ht]
\caption{LHTS Finetuning}
\label{alg:train}
\SetAlgoLined
\DontPrintSemicolon
\Indm
\KwInput{Training data $\mathcal{D}$, model $p_\phi$, temperatures $\{T_1, \ldots, T_k\}$, clip $c$, suffix horizon $h$}
\KwOutput{Parameters $\theta$ for LHTS weight-tied models $q_{T_1} \ldots q_{T_k}$}
    \Indp
    $\b, \m, n, \theta \gets {\bf 0}, {\bf 0}, 0, \phi$\\
    \While{\text{training}}{ 
        \tcc{Sample training data and temperature $T_j$}
        $x \sim \mathcal{D} \quad \quad j \sim \mathcal{U}(1,k)$\\ 
        \tcc{Compute suffix log likelihood}
        $\s \gets \texttt{RevCumSum}(\log p_\phi(x_i | x_{<i}))$\\
        \tcc{Limit suffix horizon}
        $\s \gets \s - \texttt{pad}(\s_{h:}, {\bf 0}_{:h})$\\
        \tcc{Update streaming statistics}
        $n \gets n+1 \quad \quad \b \gets \b + \s$\\
        \tcc{Compute LHTS importance weights}
        $\w \gets \exp( \text{min}( \frac{1-T_j}{T_j} (\s - \frac{\b}{n}), c) )$\\
        \tcc{Compute index-weighted log loss}
        $\mathcal{L} \gets - \sum_i \w_i \log q_{T_j,\theta}(x_i | x_{<i})$\\
        \tcc{Compute KL Loss}
        $\mathcal{K} \gets \sum_i D_{KL}( p_\phi(\cdot | x_{<i}) || q_{T_j,\theta}(\cdot | x_{<i}) )$\\
        \tcc{Normalize loss for each temperature}
        $\m_j \gets \m_j + \texttt{StopGradient}(\mathcal{L} + \beta \mathcal{K})$\\
        \tcc{Update model parameters}
        $\theta \gets \theta - \nabla_\theta (\m_j / n)^{-1} (\mathcal{L} + \beta \mathcal{K}) $
    }

\Indm
\KwReturn{ $\theta$ }
\end{algorithm2e}

\section{Experiments}

We examine LHTS on three types of models: a diffusion-based image model (DDPM), an autoregressive character model, and an autoregressive language model (GPT-2~\cite{radford2019language}). For diffusion models, we compare against a pseudo-temperature baseline that reduces the variance of the diffusion noise. For autoregressive models, we compare against myopic temperature scaling. We aim to show that LHTS can generate samples with higher likelihood and more diversity, compared to the baseline temperature scaling methods. Finally, we test the practical benefits of the LHTS language model on a multiple choice task.

\subsection{Diffusion Image Model}

First, we apply LHTS on diffusion image models. Training diffusion models from scratch can be compute intensive, so we instead take a pretrained DDPM~\cite{ho2020denoising} and finetune with the LHTS objective. We compute the ELBO of each image in the CIFAR-10~\cite{krizhevsky2009learning} training set with respect to the pretrained DDPM in \textit{joint} space (i.e.,  without dividing by the number of dimensions) and compute the LHTS importance weight of each image. Then, we finetune for $50000$ steps using the DDPM objective in Eq.~\ref{eq:ddpm_loss} scaled by the computed importance weights.

Since there is no natural temperature scaling baseline for diffusion models, we compare against a pseudo-temperature scaling of the diffusion model by reducing the noise variance. At each step of the reverse diffusion process, we sample a noise vector from a Gaussian $\mathcal{N}(0, t)$, where the pseudo-temperature $t$ is controls the standard deviation of the noise. By using a smaller noise variance, we hope to push the Langevin sampling process to converge to images with higher likelihood, though at the cost of distorting the marginal distribution at each timestep.

To compare LHTS and pseudo-temperature, we plot the log-likelihood of samples and the FID score of the temperature-scaled models. Ideally, a temperature-scaled model should output samples that evaluate as more likely under the pretrained model distribution, without sacrificing too much diversity. In Figure~\ref{fig:ddpm-pareto}, we see that the LHTS model is able to beat the Pareto frontier of pseudo-temperature scaled models, where diversity is measured with the negative FID score. We plot uncurated samples in Figure~\ref{fig:ddpm-samples}. Even though both pseudo-temperature scaling and LHTS can push the model into sampling images with higher likelihood (with respect to the the pretrained model), LHTS is able to do so while sacrificing less of generation quality and diversity.

\begin{figure}
    \includegraphics[width=\linewidth]{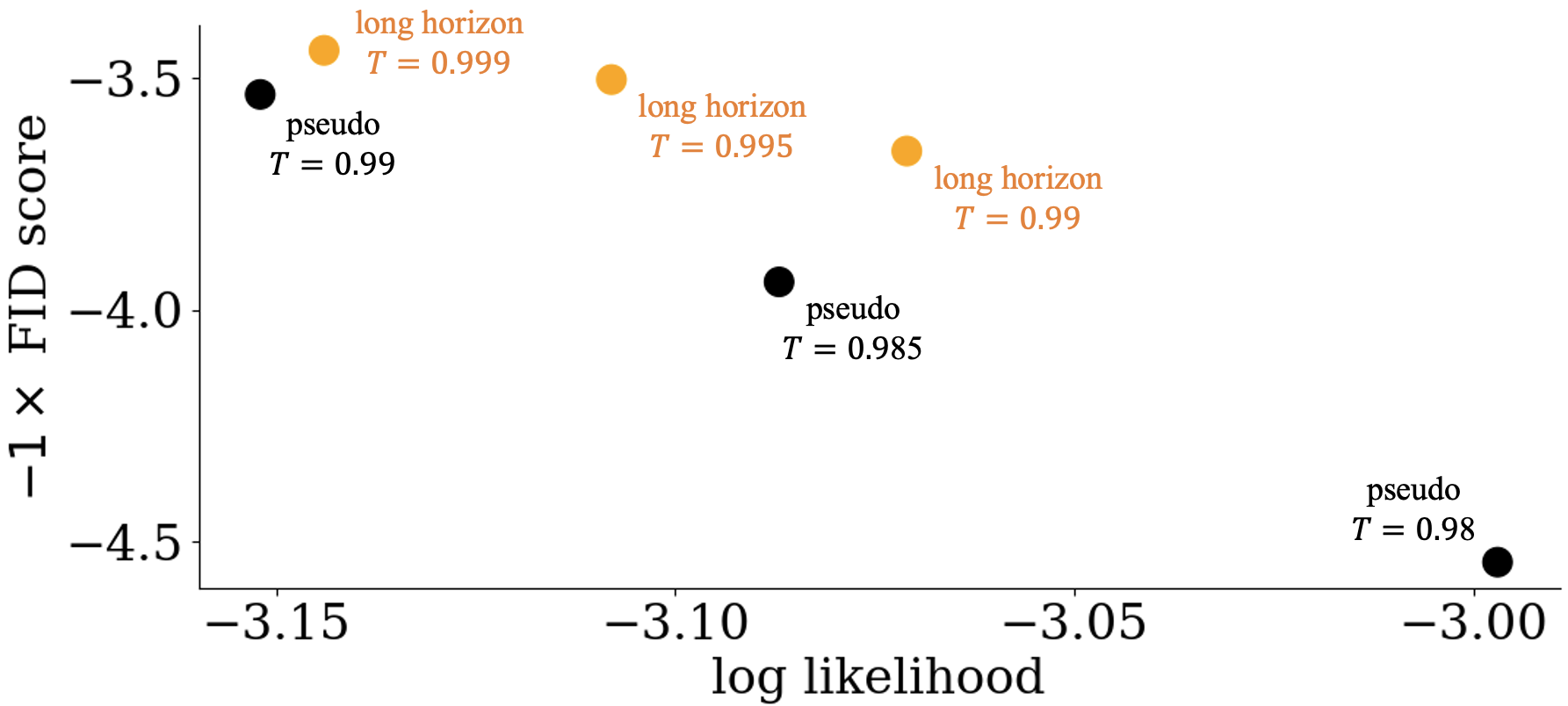}
    \caption{Temperature scaling on diffusion models for CIFAR-10. The black dots form the Pareto frontier of pseudo-temperature scaling on DDPM (with pseudo-temperatures $0.99$, $0.985$, and $0.98$), and the orange shows long horizon temperature scaling via finetuning (with long horizon temperatures $0.999$, $0.995$, $0.99$). The x-axis plots log likelihood and y-axis plots negative FID score using $50$k samples. Towards the top right of the chart is better.}
    \label{fig:ddpm-pareto}
\end{figure}

\begin{figure}
    \includegraphics[trim={0 0 0 0},clip,width=0.48\linewidth]{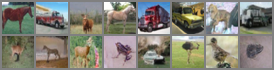} 
    \hfill
    \includegraphics[trim={0 0 0 0},clip,width=0.48\linewidth]{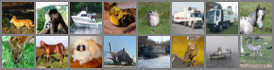}
    \caption{Generated image samples from temperature scaled DDPM. Left: pseudo-temperature scaling, with worse FID score $3.94$ and lower sample likelihood $-3.09$. Right: LHTS, with better FID score $3.66$ and higher sample likelihood $-3.07$.}
    \label{fig:ddpm-samples}
\end{figure}

\subsection{Autoregressive Character Model}
\label{sec:char-model}

Next, we experiment with a transformer-based autoregressive character model on the Text8 dataset~\cite{mahoney2011large}. Though character modeling is an easier task than language modeling, it provides useful insights on the differences between LHTS and myopic temperature scaling. In particular, it allows us to experiment with training a weight-tied model for a continuous range of long horizon temperatures.

\begin{figure}
    \includegraphics[width=\linewidth]{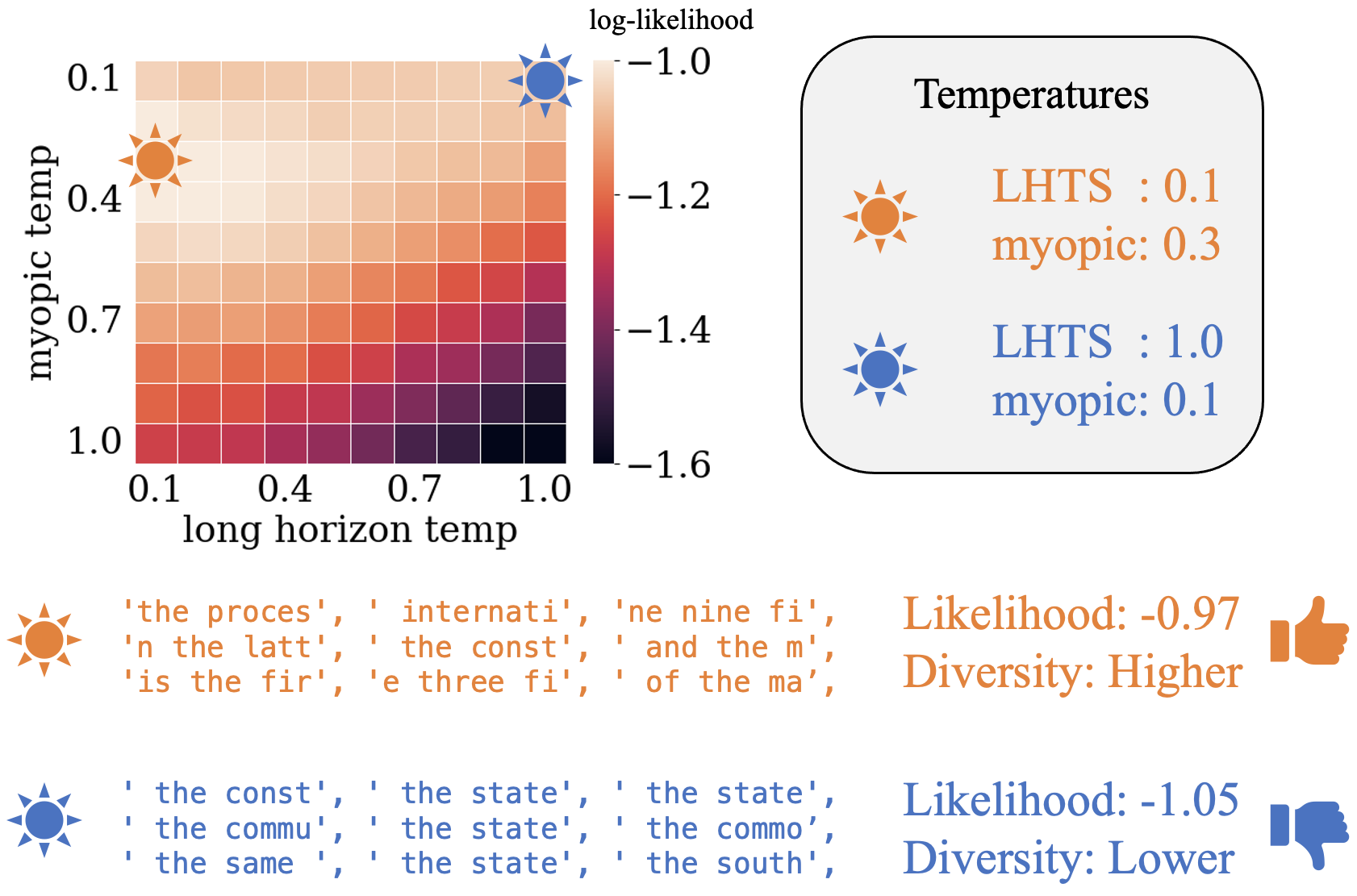}
    \caption{Autoregressive character model with a tunable long horizon temperature parameter. The heatmap shows log-likelihood of samples over various settings of long horizon and myopic temperature. Tuning both temperatures (orange) allows us to increase the likelihood more than just tuning the myopic temperature (blue). More importantly, we achieve a better trade-off between likelihood and diversity. The orange setting gives a higher likelihood with noticeably diverse chunks of text, whereas the blue setting gives lower likelihood yet gives many repetitive generations.}
    \label{fig:char}
\end{figure}

\paragraph{Continuous Temperature LHTS}
Our goal is to train a model that can be deployed with a ``knob'' for controlling the long horizon temperature, similar to how existing autoregressive models in deployment (e.g. GPT) give users control over the myopic temperature. 

We first train a base $12$-layer Transformer model from scratch, and store this model $p$ to use for computing LHTS importance weights. Then, we finetune a copy $q_T$ of this model with the LHTS objective on a continuous range of long horizon temperatures.
To do so, we place the long horizon temperature (a scalar) as a special token at the beginning of the transformer's context window. We then learn a linear embedding $r$ to map this prefix temperature token into the transformer's embedding space. Finally, we place the training data into the remaining context positions, and train with the LHTS objective. In other words, the prefix token of the long horizon temperature informs the transformer how ``sharp'' the predicted distribution should be.

Since choosing small temperatures can lead to large importance weights, we only vary the training temperature from $0.9$ to $1.1$. Nevertheless, at deployment time we can still feed temperatures beyond the training range into the learned linear embedding $r$, and push the model to extrapolate to unseen long horizon temperatures. Surprisingly, we find that the model extrapolates smoothly (Figure~\ref{fig:char}), with the sample likelihood steadily improving when feeding in long horizon temperatures much less than $0.9$ (leftward on the x-axis) into the prefix temperature token.

With a knob for the long horizon temperature, we can tweak both the long horizon and the myopic temperature in unison. For example, in Figure~\ref{fig:char} we consider two settings of tuning the long horizon temperature to $0.1$ and myopic temperature to $0.3$ (orange), versus tuning just the myopic temperature to $0.1$ (blue). The orange setting gives an average sample likelihood of $-0.97$ w.r.t. $p$, which is better than $-1.05$ w.r.t. $p$ for the blue setting. On top of that, the orange setting generates much more diverse character chunks\footnote{The model is trained on randomly cropped chunks of character, hence the samples appear to be cropped.}, whereas the blue setting repeatedly outputs the same few character chunks. This aligns with the intuition that LHTS can ``look ahead'' to find many diverse sequences of high likelihood, whereas myopic temperature scaling can only greedily choose the next token, leading to low diversity.

\subsection{Autoregressive Language Model}

Lastly, we demonstrate the scalability of LHTS on various sizes of the GPT-2 (small, medium, large) language model. As before, we take a pretrained model to compute LHTS importance weights, and finetune a copy of it using the LHTS objective.
We use the standard GPT-2 architecture and context window of $1024$, with pretrained weights from HuggingFace~\cite{wolf-etal-2020-transformers}, and finetune on the OpenWebText~\cite{Gokaslan2019OpenWeb} corpus.

We compare with two baselines: myopic temperature scaling, and a partition-based controllable generation approach (Quark)~\cite{lu2022quark}. Quark was introduced as a conditional generation approach for controlling the level of toxicity of the language model, but can be similarly applied for controlling sample likelihood of the model.

When comparing different approaches for temperature scaling, we consider both the likelihood and the quality of the generated samples. We can measure the likelihood of the generated samples by directly evaluating them on the pretrained model. For sample quality, and we rely on quantitative evaluation using diversity metrics~\cite{welleck2019neural, liang2022holistic} and a multiple-choice task~\cite{mikolov2013efficient}, where temperature scaling is commonly used to reduce the randomness of the model's answers.

In Figure~\ref{fig:llm-pareto}, we plot the sample diversity (measured by token-level repetition) and the log-likelihood over 1k sequences of context $1024$ for each temperature scale. In each of the three charts (for GPT2 small/medium/large), we see the Pareto frontier of circles formed by the myopic temperature scaling baseline with temperature ranging from $0.75$ to $0.8$. Using LHTS, plotted by triangles, we can achieve a better trade-off between diversity and likelihood, especially for GPT2-medium and GPT2-large. The parition (Quark) baseline is not visible since the repetition values are worse and do not reside within the bounds of the chart.

\begin{figure*}[t]
    \begin{subfigure}{.73\linewidth}
        \includegraphics[width=0.32\linewidth]{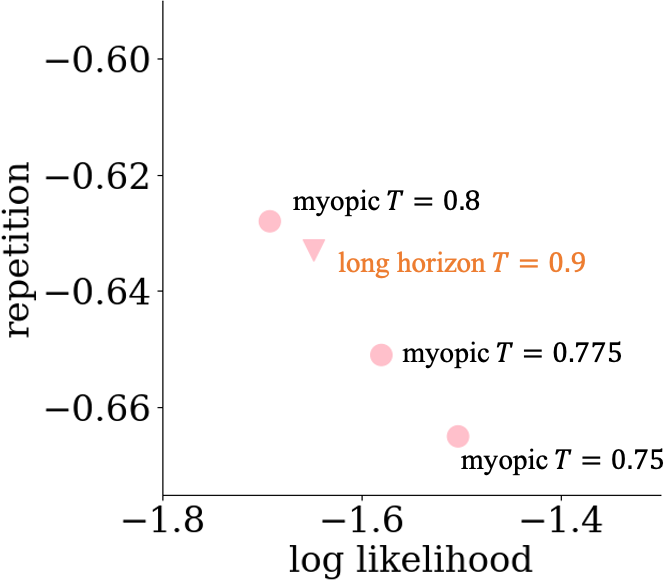}
        \hfill
        \includegraphics[width=0.32\linewidth]{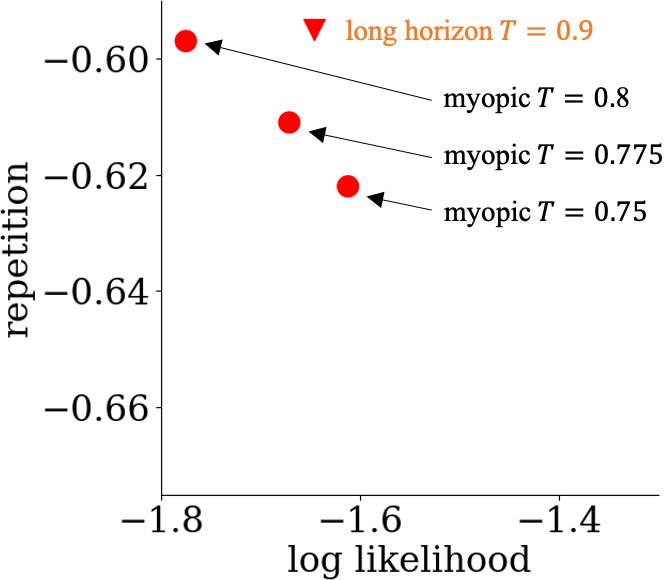}
        \hfill
        \includegraphics[width=0.32\linewidth]{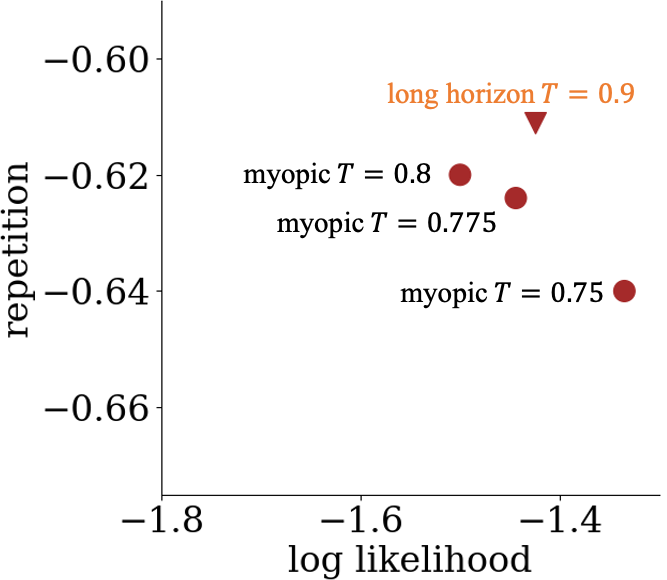}
        \caption{Plot of the Pareto frontier between repetitiveness of text and log-likelihood of text for GPT-2 models: small (pink), medium (red), large (maroon). Circles show myopic temperatures of $0.8$, $0.775$, $0.75$. Triangles show LHTS temperature of $0.95$. Towards the top right of the charts is better.}
        \label{fig:llm-pareto}
    \end{subfigure}
    \hfill
    \begin{subfigure}{.24\linewidth}
        \includegraphics[width=\linewidth]{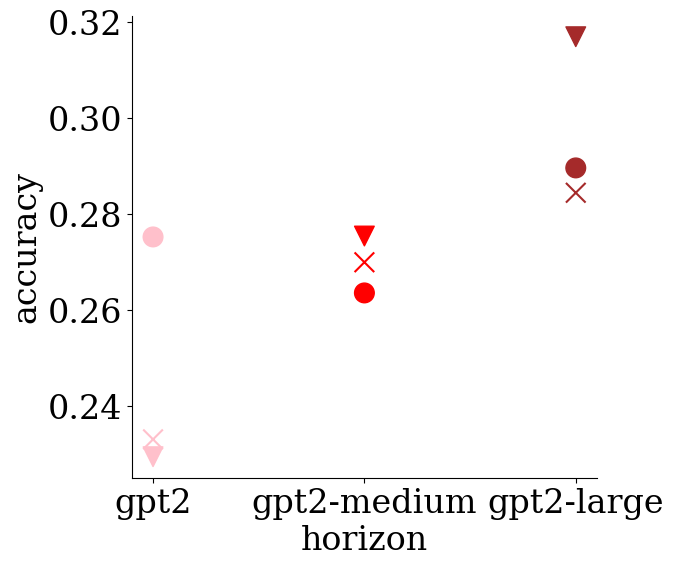}
        \caption{Visualizing the best settings from Table~\ref{tab:analogy}. Circle: myopic, Cross: Quark, Triangle: LHTS}
        \label{fig:analogy}
    \end{subfigure}
    \caption{Likelihood and sample quality metrics for temperature-scaled GPT-2.}
\end{figure*}

\begin{table*}[ht!]
    \caption{Accuracy of temperature-scaled GPT-2 on a multiple choice analogy task. Turning the myopic temperature down decreases the chance of irrelevant answers. At the lowest myopic temperature, LHTS generally improves upon the accuracy of the pretrained model.}
    \label{tab:analogy}
    \centering
    \begin{tabular}{cc|ccc|ccc|ccc}
    \toprule
     & model & \multicolumn{3}{c}{gpt2 small} & \multicolumn{3}{c}{gpt2 medium} & \multicolumn{3}{c}{gpt2 large}\\
     & myopic $T$ & 1.0 & 0.5 & 0.0 & 1.0 & 0.5 & 0.0 & 1.0 & 0.5 & 0.0 \\
    \midrule
     \multirow{ 3 }{*}{distinct}
     & LHTS $T=0.95$ & 0.172 & 0.238 & \textbf{0.255} & 0.185 & 0.242 & \textbf{0.252} & 0.179 & 0.225 & 0.232\\
     & pretrained & 0.143 & 0.231 & \textbf{0.254} & 0.156 & 0.220 & 0.233 & 0.142 & 0.218 & 0.228\\
     & partition (Quark) & 0.111 & 0.201 & 0.233 & 0.155 & 0.219 & 0.232 & 0.158 & 0.229 & \textbf{0.250}\\ 
    \midrule
     \multirow{ 3 }{*}{duplicate}
     & LHTS $T=0.95$ & 0.177 & 0.224 & 0.230 & 0.225 & 0.270 & \textbf{0.275} & 0.249 & 0.310 & \textbf{0.317}\\
     & pretrained & 0.189 & 0.267 & \textbf{0.275} & 0.200 & 0.262 & 0.264 & 0.203 & 0.279 & 0.290\\
     & partition (Quark) & 0.137 & 0.221 & 0.233 & 0.197 & 0.264 & 0.270 & 0.213 & 0.279 & 0.285\\
    \bottomrule
    \end{tabular}
\end{table*}

\paragraph{Analogy Multiple Choice} We evaluate the generation quality of LHTS on a downstream multiple-choice task that tests the model's ability to choose correct analogies. We create a set of $1400$ questions from a bank of analogies~\cite{mikolov2013efficient} with relationships such as \textit{country:capital}, \textit{present-tense:past-tense}, \textit{male:female}.
We prompt GPT-2 using the following format, including three similar examples in-context:\\
{\scriptsize \textit {
Question: Please choose the word pair that is most analogous to ``Algeria dinar''.\\
Choices: ``Macedonia dollar'', ``Vietnam baht'', ``Bulgaria lev'', ``Armenia naira''\\
Answer: 
}}

To measure correctness, we check the next $8$ generated tokens for a unique match with the correct choice, ignoring double matches. We also create a variant of questions where three of the choices share the first word, inspired by the example in Figure~\ref{fig:cover}. The three duplicates are chosen independently from (and can include) the correct choice.\\
{\scriptsize \textit {
Question: Please choose the word pair that is most analogous to ``Athens Greece''.\\
Choices: ``Moscow Japan'', ``Rome Italy'', ``Moscow Pakistan'', ``Moscow Australia''\\
Answer: 
}}

In Table~\ref{tab:analogy} we present the accuracy of GPT-2 on this analogy multiple-choice task. The row \textit{distinct} refers to the first set of questions, and the row \textit{duplicate} refers to the second set of questions with common first words. For each question we sample the model $50$ times. 
The accuracy improves across the board as we scale down the myopic temperature from $1.0$ to $0.0$, since all models reduce the chance of outputting irrelevant answers. At the best myopic temperature of $0.0$, LHTS gives the highest accuracy in 4/6 settings, with $10\%$ improvement to give $31\%$ accuracy on the most competitive setting with the \textit{duplicate} question set and GPT2 large. We also see that using Quark to condition on joint likelihood is less effective on average, possibly because partitioning full sequences based on joint likelihood is more crude than a suffix-dependent rescaling (Section~\ref{sec:variance-reduced-lhts}), and reinforcing based on likelihood was noticeably unstable during training.

\section{Conclusion}

We present Long Horizon Temperature Scaling, a novel and tractable approach to sampling from temperature-scaled joint distributions. Compared to previous methods, LHTS is non-myopic and compatible with all likelihood-based generative models. To reduce the variance of the LHTS objective, we introduce important techniques such as fitting baselines and limiting suffix horizon lengths. In some settings, LHTS even shows smooth extrapolation to unseen temperatures, enabling low-variance training on mild temperatures and sampling on extreme temperatures.
We demonstrate the applicability of LHTS on diffusion and autoregressive models in image and language domains. LHTS shows improvements over pseudo and myopic temperature scaling in the trade-off between likelihood and sample diversity, and in the accuracy of a multiple-choice analogy task.

\textbf{Limitations and Future Work} \, Temperature scaling the joint distribution is inherently intractable (scaling the temperature to $0$ gives the argmax of the joint distribution), and LHTS only aims to learn an approximation to the solution. In addition, LHTS involves finetuning the model, as opposed to pure post-hoc alternatives such as myopic or pseudo temperature scaling. Future work can look into exploring multi-temperature finetuning further, or other divergences besides forward-KL for the LHTS objective.

\section{Acknowledgments}

We thank anonymous reviewers for their constructive feedback. This research was supported in part by NSF (\#1651565), ARO (W911NF-21-1-0125), ONR (N00014-23-1-2159, N00014-22-1-2293), CZ Biohub, HAI.

\bibliography{ref}

\begin{thebibliography}{28}
\providecommand{\natexlab}[1]{#1}
\providecommand{\url}[1]{\texttt{#1}}
\expandafter\ifx\csname urlstyle\endcsname\relax
  \providecommand{\doi}[1]{doi: #1}\else
  \providecommand{\doi}{doi: \begingroup \urlstyle{rm}\Url}\fi

\bibitem[Bommasani et~al.(2021)Bommasani, Hudson, Adeli, Altman, Arora, von
  Arx, Bernstein, Bohg, Bosselut, Brunskill,
  et~al.]{bommasani2021opportunities}
Bommasani, R., Hudson, D.~A., Adeli, E., Altman, R., Arora, S., von Arx, S.,
  Bernstein, M.~S., Bohg, J., Bosselut, A., Brunskill, E., et~al.
\newblock On the opportunities and risks of foundation models.
\newblock \emph{arXiv preprint arXiv:2108.07258}, 2021.

\bibitem[Brown et~al.(2020)Brown, Mann, Ryder, Subbiah, Kaplan, Dhariwal,
  Neelakantan, Shyam, Sastry, Askell, et~al.]{brown2020language}
Brown, T., Mann, B., Ryder, N., Subbiah, M., Kaplan, J.~D., Dhariwal, P.,
  Neelakantan, A., Shyam, P., Sastry, G., Askell, A., et~al.
\newblock Language models are few-shot learners.
\newblock \emph{Advances in neural information processing systems},
  33:\penalty0 1877--1901, 2020.

\bibitem[Clark et~al.(2019)Clark, Lee, Chang, Kwiatkowski, Collins, and
  Toutanova]{clark2019boolq}
Clark, C., Lee, K., Chang, M.-W., Kwiatkowski, T., Collins, M., and Toutanova,
  K.
\newblock Boolq: Exploring the surprising difficulty of natural yes/no
  questions.
\newblock \emph{arXiv preprint arXiv:1905.10044}, 2019.

\bibitem[Desai \& Durrett(2020)Desai and Durrett]{desai2020calibration}
Desai, S. and Durrett, G.
\newblock Calibration of pre-trained transformers.
\newblock \emph{arXiv preprint arXiv:2003.07892}, 2020.

\bibitem[Fan et~al.(2018)Fan, Lewis, and Dauphin]{fan2018hierarchical}
Fan, A., Lewis, M., and Dauphin, Y.
\newblock Hierarchical neural story generation.
\newblock In \emph{Proceedings of the 56th Annual Meeting of the Association
  for Computational Linguistics (Volume 1: Long Papers)}, pp.\  889--898, 2018.

\bibitem[Gershman \& Goodman(2014)Gershman and Goodman]{gershman2014amortized}
Gershman, S. and Goodman, N.
\newblock Amortized inference in probabilistic reasoning.
\newblock In \emph{Proceedings of the annual meeting of the cognitive science
  society}, volume~36, 2014.

\bibitem[Gokaslan \& Cohen(2019)Gokaslan and Cohen]{Gokaslan2019OpenWeb}
Gokaslan, A. and Cohen, V.
\newblock Openwebtext corpus, 2019.

\bibitem[Guo et~al.(2017)Guo, Pleiss, Sun, and Weinberger]{guo2017calibration}
Guo, C., Pleiss, G., Sun, Y., and Weinberger, K.~Q.
\newblock On calibration of modern neural networks.
\newblock In \emph{International conference on machine learning}. PMLR, 2017.

\bibitem[Ho et~al.(2020)Ho, Jain, and Abbeel]{ho2020denoising}
Ho, J., Jain, A., and Abbeel, P.
\newblock Denoising diffusion probabilistic models.
\newblock \emph{Advances in Neural Information Processing Systems},
  33:\penalty0 6840--6851, 2020.

\bibitem[Holtzman et~al.(2019)Holtzman, Buys, Du, Forbes, and
  Choi]{holtzman2019curious}
Holtzman, A., Buys, J., Du, L., Forbes, M., and Choi, Y.
\newblock The curious case of neural text degeneration.
\newblock \emph{arXiv preprint arXiv:1904.09751}, 2019.

\bibitem[Kingma \& Dhariwal(2018)Kingma and Dhariwal]{kingma2018glow}
Kingma, D.~P. and Dhariwal, P.
\newblock Glow: Generative flow with invertible 1x1 convolutions.
\newblock \emph{Advances in neural information processing systems}, 31, 2018.

\bibitem[Koller \& Friedman(2009)Koller and Friedman]{koller2009probabilistic}
Koller, D. and Friedman, N.
\newblock \emph{Probabilistic graphical models: principles and techniques}.
\newblock MIT press, 2009.

\bibitem[Krizhevsky et~al.(2009)Krizhevsky, Hinton,
  et~al.]{krizhevsky2009learning}
Krizhevsky, A., Hinton, G., et~al.
\newblock Learning multiple layers of features from tiny images, 2009.

\bibitem[Kwiatkowski et~al.(2019)Kwiatkowski, Palomaki, Redfield, Collins,
  Parikh, Alberti, Epstein, Polosukhin, Devlin, Lee,
  et~al.]{kwiatkowski2019natural}
Kwiatkowski, T., Palomaki, J., Redfield, O., Collins, M., Parikh, A., Alberti,
  C., Epstein, D., Polosukhin, I., Devlin, J., Lee, K., et~al.
\newblock Natural questions: a benchmark for question answering research.
\newblock \emph{Transactions of the Association for Computational Linguistics},
  7:\penalty0 453--466, 2019.

\bibitem[Li et~al.(2016)Li, Monroe, and Jurafsky]{li2016simple}
Li, J., Monroe, W., and Jurafsky, D.
\newblock A simple, fast diverse decoding algorithm for neural generation.
\newblock \emph{arXiv preprint arXiv:1611.08562}, 2016.

\bibitem[Liang et~al.(2022)Liang, Bommasani, Lee, Tsipras, Soylu, Yasunaga,
  Zhang, Narayanan, Wu, et~al.]{liang2022holistic}
Liang, P., Bommasani, R., Lee, T., Tsipras, D., Soylu, D., Yasunaga, M., Zhang,
  Y., Narayanan, D., Wu, Y., et~al.
\newblock Holistic evaluation of language models.
\newblock \emph{arXiv preprint arXiv:2211.09110}, 2022.

\bibitem[Lu et~al.(2022)Lu, Welleck, Jiang, Hessel, Qin, West, Ammanabrolu, and
  Choi]{lu2022quark}
Lu, X., Welleck, S., Jiang, L., Hessel, J., Qin, L., West, P., Ammanabrolu, P.,
  and Choi, Y.
\newblock Quark: Controllable text generation with reinforced unlearning.
\newblock \emph{arXiv preprint arXiv:2205.13636}, 2022.

\bibitem[Mahoney(2011)]{mahoney2011large}
Mahoney, M.
\newblock Large text compression benchmark, 2011.

\bibitem[Mikolov et~al.(2013)Mikolov, Chen, Corrado, and
  Dean]{mikolov2013efficient}
Mikolov, T., Chen, K., Corrado, G., and Dean, J.
\newblock Efficient estimation of word representations in vector space.
\newblock \emph{arXiv preprint arXiv:1301.3781}, 2013.

\bibitem[Narayan et~al.(2018)Narayan, Cohen, and Lapata]{narayan2018don}
Narayan, S., Cohen, S.~B., and Lapata, M.
\newblock Don't give me the details, just the summary! topic-aware
  convolutional neural networks for extreme summarization.
\newblock \emph{arXiv preprint arXiv:1808.08745}, 2018.

\bibitem[Nichol \& Dhariwal(2021)Nichol and Dhariwal]{nichol2021improved}
Nichol, A.~Q. and Dhariwal, P.
\newblock Improved denoising diffusion probabilistic models.
\newblock In \emph{International Conference on Machine Learning}, pp.\
  8162--8171. PMLR, 2021.

\bibitem[Nixon et~al.(2019)Nixon, Dusenberry, Zhang, Jerfel, and
  Tran]{nixon2019measuring}
Nixon, J., Dusenberry, M.~W., Zhang, L., Jerfel, G., and Tran, D.
\newblock Measuring calibration in deep learning.
\newblock In \emph{CVPR Workshops}, volume~2, 2019.

\bibitem[Pillutla et~al.(2021)Pillutla, Swayamdipta, Zellers, Thickstun,
  Welleck, Choi, and Harchaoui]{pillutla2021mauve}
Pillutla, K., Swayamdipta, S., Zellers, R., Thickstun, J., Welleck, S., Choi,
  Y., and Harchaoui, Z.
\newblock Mauve: Measuring the gap between neural text and human text using
  divergence frontiers.
\newblock \emph{Advances in Neural Information Processing Systems},
  34:\penalty0 4816--4828, 2021.

\bibitem[Radford et~al.(2019)Radford, Wu, Child, Luan, Amodei, Sutskever,
  et~al.]{radford2019language}
Radford, A., Wu, J., Child, R., Luan, D., Amodei, D., Sutskever, I., et~al.
\newblock Language models are unsupervised multitask learners.
\newblock \emph{OpenAI blog}, 1\penalty0 (8):\penalty0 9, 2019.

\bibitem[Vahdat \& Kautz(2020)Vahdat and Kautz]{vahdat2020nvae}
Vahdat, A. and Kautz, J.
\newblock Nvae: A deep hierarchical variational autoencoder.
\newblock \emph{Advances in Neural Information Processing Systems},
  33:\penalty0 19667--19679, 2020.

\bibitem[Vijayakumar et~al.(2018)Vijayakumar, Cogswell, Selvaraju, Sun, Lee,
  Crandall, and Batra]{vijayakumar2018diverse}
Vijayakumar, A., Cogswell, M., Selvaraju, R., Sun, Q., Lee, S., Crandall, D.,
  and Batra, D.
\newblock Diverse beam search for improved description of complex scenes.
\newblock In \emph{Proceedings of the AAAI Conference on Artificial
  Intelligence}, volume~32, 2018.

\bibitem[Welleck et~al.(2019)Welleck, Kulikov, Roller, Dinan, Cho, and
  Weston]{welleck2019neural}
Welleck, S., Kulikov, I., Roller, S., Dinan, E., Cho, K., and Weston, J.
\newblock Neural text generation with unlikelihood training.
\newblock \emph{arXiv preprint arXiv:1908.04319}, 2019.

\bibitem[Wolf et~al.(2020)Wolf, Debut, Sanh, Chaumond, Delangue, Moi, Cistac,
  Rault, Louf, Funtowicz, Davison, Shleifer, von Platen, Ma, Jernite, Plu, Xu,
  Scao, Gugger, Drame, Lhoest, and Rush]{wolf-etal-2020-transformers}
Wolf, T., Debut, L., Sanh, V., Chaumond, J., Delangue, C., Moi, A., Cistac, P.,
  Rault, T., Louf, R., Funtowicz, M., Davison, J., Shleifer, S., von Platen,
  P., Ma, C., Jernite, Y., Plu, J., Xu, C., Scao, T.~L., Gugger, S., Drame, M.,
  Lhoest, Q., and Rush, A.~M.
\newblock Transformers: State-of-the-art natural language processing.
\newblock In \emph{Proceedings of the 2020 Conference on Empirical Methods in
  Natural Language Processing: System Demonstrations}, October 2020.

\end{thebibliography}
\bibliographystyle{icml2023}

\newpage
\appendix
\onecolumn
\section{Experimental Settings}

\paragraph{Diffusion Model}

\begin{itemize}
    \item Architecture: DDPM
    \item Learning Rate: 2e-4
    \item Batch Size: 128
    \item EMA decay: 0.9999
    \item Grad Clip: 1
    \item Steps: 50000
    \item Warmup Steps: 5000
    \item LHTS Clip: 0.5
\end{itemize}

\paragraph{Character Model}

\begin{itemize}
    \item Architecture: 12-layer Transformer, embedding 768, hidden size 3072, num heads 12, num layers 12
    \item Learning Rate: 5e-4
    \item Batch Size: 512
    \item Weight Decay: 0.001
    \item Grad Clip: 0.25
    \item Epochs: 200
    \item LHTS Clip: 3
    \item LHTS Suffix Horizon: 25
\end{itemize}

\paragraph{Language Model}

\begin{itemize}
    \item Architecture: GPT-2 (small, medium, large), context 1024
    \item Learning Rate: 1e-4
    \item Batch Size: 512
    \item Weight Decay: 0.01
    \item Grad Clip: 0.25
    \item Steps: 1000
    \item LHTS KL beta: 0.05
    \item LHTS Clip: 3
    \item LHTS Suffix Horizon: 8
\end{itemize}

\section{Example Sample from GPT-2 Large with LHTS}
\textit{It is always great when you get a chance to get the inside scoop as to why a franchise is so popular, and today I have learned that one of the reasons is the players they play them against. When I joined the St Louis Cardinals, one of my first observations of Albert Pujols was that he was a little tough to hit. Now, having played him a ton of baseball, there's no doubt in my mind that he's just as tough, if not tougher, than anybody else on the field. Pujols, like most power hitters before him, is known best for running his bat out all over the field. The one thing he has in his favor, though, is that when he does come back to the field, he always finds his way to hitting lefties.}

\section{Additional Experiments}

We evaluate the diffusion models and language models on additional metrics such as SSIM, MAUVE score, and HELM benchmarks.

\subsection{Diffusion Model}

We use the same DDPM diffusion model from Figure~\ref{fig:ddpm-pareto}, finetuned with LHTS, and report the Structural Similarity Index (SSIM). Unlike FID, SSIM does not consider diversity, but rather closeness to a ground-truth image. Therefore, we see that the SSIM scores in Table~\ref{tab:ssim} align roughly with the log-likelihood scores of Figure~\ref{fig:ddpm-pareto}.

\begin{table}[h]
    \centering
    \caption{SSIM of diffusion model with pseudo-temperature scaling and long horizon temperature scaling on CIFAR-10.}
    \label{tab:ssim}
    \begin{tabular}{c|ccc|ccc}
        \toprule
         & \multicolumn{3}{|c|}{PseudoTemp} & \multicolumn{3}{|c}{LHTS} \\
         Temperature & 0.98 & 0.985 & 0.99 & 0.99 & 0.995 & 0.999  \\
         \midrule
         SSIM & 0.915 & 0.913 & 0.913 & 0.913 & 0.913 & 0.911 \\
         \bottomrule
    \end{tabular}
\end{table}

\subsection{Language Model}

We examine the same GPT-2 language model from Figure~\ref{fig:llm-pareto} and Table~\ref{tab:analogy}, using the gpt2-large size. We compare the use of standard myopic temperature scaling versus LHTS finetuning on MAUVE score~\cite{pillutla2021mauve} and a number of metrics from HELM~\cite{liang2022holistic}.

\paragraph{MAUVE score}

We evaluate MAUVE score on OpenWebText~\cite{Gokaslan2019OpenWeb} using the setup in the MAUVE paper~\cite{pillutla2021mauve}, with 1000 generations and a prompt length of 30 tokens. We find that LHTS does not improve MAUVE score, and that both forms of temperature scaling (myopic and LHTS) in general decrease MAUVE score.

\begin{table}[h]
    \centering
    \caption{MAUVE score of GPT-2 (gpt2-large) with myopic temperature scaling and long horizon temperature scaling on OpenWebText.}
    \label{tab:mauve}
    \begin{tabular}{c|c|cc|cc}
        \toprule
         & No Scaling & \multicolumn{2}{|c|}{Myopic Only} & \multicolumn{2}{|c}{LHTS} \\
         Myopic Temperature & 1.0 & 0.8 & 0.0 & 1.0 & 0.0 \\
         Long Horizon Temperature & 1.0 & 1.0 & 1.0 & 0.9 & 0.9 \\
         \midrule
         MAUVE & 0.76 & 0.57 & 0.00 & 0.41 & 0.00 \\
         \bottomrule
    \end{tabular}
\end{table}

\newpage

\paragraph{HELM}

We evaluate on some metrics from the HELM benchmark such as XSUM~\cite{narayan2018don}, BoolQ~\cite{clark2019boolq}, and NaturalQA open book~\cite{kwiatkowski2019natural}, which respectively test for summarization, classification, and reading comprehension with multiple choice.

\begin{table}[h]
    \centering
    \caption{Evaluation of GPT-2 (gpt2-large) on various metrics from HELM.}
    \label{tab:helm}
    \begin{tabular}{c|cc|cc}
        \toprule
         & \multicolumn{2}{|c|}{Myopic Only} & \multicolumn{2}{|c}{LHTS} \\
         Myopic Temperature & 0.8 & 0.0 & 1.0 & 0.0 \\
         Long Horizon Temperature & 1.0 & 1.0 & 0.9 & 0.9 \\
         \midrule
xsum\_test\_rouge2 & 0.016 & 0.019 & 0.013 & 0.02 \\
xsum\_test\_perp (lower) & 6.72 & 2.305 & 5.137 & 1.725 \\
\midrule
boolq\_exact\_match & 0.383 & 0.583 & 0.417 & 0.603 \\
boolq\_exact\_match\_fairness & 0.167 & 0.483 & 0.173 & 0.507 \\
boolq\_exact\_match\_robustness & 0.087 & 0.35 & 0.113 & 0.43 \\
boolq\_ece\_10\_bin (lower) & 0.112 & 0.164 & 0.124 & 0.174 \\
\midrule
naturalqa\_open\_f1\_score & 0.157 & 0.257 & 0.146 & 0.248 \\
naturalqa\_open\_f1\_score\_fairness & 0.058 & 0.153 & 0.041 & 0.164 \\
naturalqa\_open\_f1\_score\_robustness & 0.026 & 0.074 & 0.031 & 0.055 \\
naturalqa\_open\_ece\_10\_bin (lower) & 0.109 & 0.134 & 0.086 & 0.14 \\
         \bottomrule
    \end{tabular}
\end{table}

In Table~\ref{tab:helm}, we can see that LHTS shows some improvements in perplexity and accuracy, and less so for calibration and F1-score. LHTS can also help with fairness or robustness perturbations on both accuracy and F1-score. However, we note that these scores are generally low due to the relatively small size of GPT-2, so evaluations on larger models are needed for more conclusive results.

\end{document}